\documentclass[a4paper,fleqn]{cas-sc}

\usepackage[natbibapa]{apacite}

\usepackage{stmaryrd}
\usepackage{caption}
\usepackage{subcaption}
\usepackage{xcolor}
\usepackage{todonotes}

\def\tsc#1{\csdef{#1}{\textsc{\lowercase{#1}}\xspace}}
\tsc{WGM}
\tsc{QE}

\newcommand{\err}{\text{err}}
\newcommand{\Err}{\text{Err}}
\newcommand{\eps}{\varepsilon}
\newcommand{\Prob}{\mathbb{P}}

\newcommand{\confidence}{\gamma}
\newcommand{\credibility}{c}
\newcommand{\bagname}{\Sigma}

\newproof{proof}{Proof}
\newtheorem{proposition}{Proposition}
\newtheorem{remark}{Remark}

\begin{document}
\let\WriteBookmarks\relax
\def\floatpagepagefraction{1}
\def\textpagefraction{.001}

\shorttitle{}    

\shortauthors{}  

\title [mode = title]{Classification with Reject Option: Distribution-free Error Guarantees via Conformal Prediction}  



%

\author[1,3]{Johan {Hallberg Szabadv\'ary}}[orcid=0000-0001-7240-8919]

\cormark[1]

\fnmark[1]

\ead{johan.hallberg.szabadvary@ju.se}

\credit{}

\affiliation[1]{organization={Department of Computing, J\"onk\"oping University},
            country={Sweden}}
\affiliation[2]{organization={Department of Computer Science, Royal Holloway, University of London},
            country={UK}}
\affiliation[3]{organization={Department of Mathematics, Stockholm University},
            country={Sweden}}
\affiliation[4]{organization={M\"olnlycke Health Care AB, Gothenburg},
            country={Sweden}}

\author[1]{Tuwe L\"ofstr\"om}[orcid=0000-0003-0274-9026]
\ead{tuwe.lofstrom@ju.se}

\author[1]{Ulf Johansson}[orcid=0000-0003-0412-6199]
\ead{ulf.johansson@ju.se}

\author[1]{Cecilia Sönströd}[orcid=0009-0009-0404-2586]
\ead{cecilia.sonstrod@ju.se}

\author[2,4]{Ernst Ahlberg}[orcid=0000-0003-2050-9069]
\ead{ernst.ahlberg@molnlycke.com}

\author[1,2]{Lars Carlsson}[orcid=0000-0001-9491-4134]
\ead{lars.carlsson@ju.se}

\cortext[1]{Corresponding author}


\begin{abstract}
    Machine learning (ML) models always make a prediction, even when they are likely to be wrong. This causes problems in practical applications, as we do not know if we should trust a prediction. ML with reject option addresses this issue by abstaining from making a prediction if it is likely to be incorrect. In this work, we formalise the approach to ML with reject option in binary classification, deriving theoretical guarantees on the resulting error rate. This is achieved through conformal prediction (CP), which produce prediction sets with distribution-free validity guarantees. In binary classification, CP can output prediction sets containing exactly one, two or no labels. By accepting only the singleton predictions, we turn CP into a binary classifier with reject option. 

    Here, CP is formally put in the framework of predicting with reject option. We state and prove the resulting error rate, and give finite sample estimates.
    Numerical examples provide illustrations of derived error rate through several different conformal prediction settings, ranging from full conformal prediction to offline batch inductive conformal prediction. The former has a direct link to sharp validity guarantees, whereas the latter is more fuzzy in terms of validity guarantees but can be used in practice. Error-reject curves illustrate the trade-off between error rate and reject rate, and can serve to aid a user to set an acceptable error rate or reject rate in practice.
\end{abstract}

\begin{graphicalabstract}
\includegraphics[width=\textwidth]{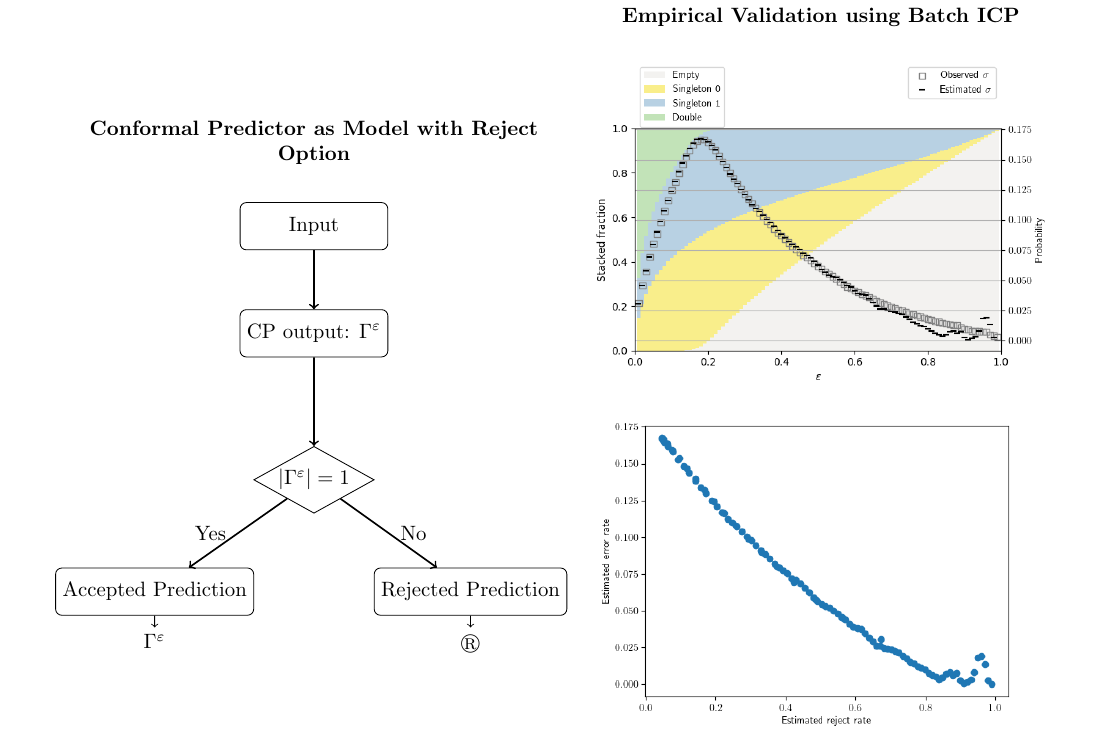}
\end{graphicalabstract}

\begin{highlights}
\item 
Proven Distribution-Free Error Rates for Singleton Predictions: Delivered theoretical guarantees on error rates for accepted singleton predictions, resolving inaccuracies found in prior works.

\item 
Refinement of Reject Option in Conformal Prediction: Clarified the relationship between error rate and reject rate using error-reject curves, offering a more accurate approach than existing models.

\item 
Practical Implementation Guidelines for Reject Option Models: Presented clear practical insights for implementing conformal predictors with reject options, providing empirical guidance for managing trade-offs in real-world applications.

\item 
Corrected Error Rate Analysis for Conformal Prediction with Reject Option: Provided a rigorous and corrected analysis of the error rate for binary classification with reject option, addressing weaknesses in previous studies.
\end{highlights}


\maketitle

\section{Introduction}\label{section:Introduction}

Traditional machine learning (ML) models always make a prediction, even when it is likely that the prediction will be incorrect. This is often undesirable, as in the case of decision support applications, where we want to be able to trust the information we are given. Machine learning with reject option aims to avoid making predictions when they are likely to be wrong, increasing the trustworthiness of a model. Most of this work trace back to~\citet{chow1957optimum} and~\citet{chow1970optimum}. In these papers a theoretical framework is established for optimum rejection rules and in particular the relation between rejection rate and the error rate for the accepted predictions. A rejection is achieved by allowing the model to abstain from giving predictions for some objects. The rejected objects can further be divided into two categories, ambiguity rejection and novelty rejection. For a more comprehensive description, see \citet{hendrickx2024machine}.

The main contribution of this work, is a method for ML with reject option with a provable distribution-free error rate. While optimal strategies for reject option has been studied in e.g. \citet{franc2023optimal}, \citet{cortes2016learning}, and \citet{hendrickx2024machine}, no strategy to the best of our knowledge provides a distribution-free error rate. We achieve this through Conformal prediction (CP) \citep{alrw2}, a method to generate prediction sets with guarantees on the error rate, although defined for set predictions. The key property of conformal predictors is that their predictions are valid, meaning, in the strongest formulation, that errors, i.e. the true label is not included in the prediction set, are independent, and happen with a user specified probability. In recent years, work has been done suggesting that validity is supported for singleton predictions, {i.e.} just one label is within the prediction set, see {e.g.} \citet{linusson2016reliable}, \citet{linusson18} and \citet{johansson23}. By rejecting all but the singleton predictions, one can view conformal prediction as a case of ML with reject option. For another example of this, see e.g. \citet{bortolussi2019conformal}. Common for these works is that they each present different procedures for creating singleton predictions with an associated error rate and empirical results have been presented to support the different procedures, respectively. However, none of the above works use the correct singleton error rates, and different formulas are used in different papers, which leads to confusion about what is theoretically justifiable. The formula appearing in \citet{linusson2016reliable} and \citet{bortolussi2019conformal} is close to being correct, but since both works consider CP in the inductive, offline setting, the overall error rate must be adjusted accordingly, which they fail to account for, making their singleton error rate approximately correct only for large calibration sets. We present the relevant theory in Section \ref{ICP theory}, and the correct singleton error rate in Section \ref{sec:ICP singleton}, which is of the PAC type, including two parameters.

This work resolves the confusion by providing the theoretical singleton error rates in the online setting (where the formula appearing in \citet{linusson2016reliable} actually holds) as well as the offline setting, with the relevant modifications to account for training conditional validity. 
Moreover, we formally describe the model with reject option that results from using CP in this mode, aligning the notation with that of \citet{hendrickx2024machine}. A self-contained introduction to conformal prediction as well as statements of the error rates in different scenarios are given in the appendices.

The rest of the paper is structured as follows. We give a brief introduction to classification with regret option in Section \ref{sec:rejectIntro}, followed by an informal introduction to conformal prediction in Section \ref{section:A Brief Introduction to Conformal Prediction}. Section \ref{section:Conformal Predictors as Rejectors} makes precise the connection between ML with reject option and CP, defining formally the CP rejector, which abstains from making predictions whenever the prediction set is not a singleton. We then move on to prove the error rate of CP rejectors in Section \ref{section:Error Rate}. Section \ref{section:Numerical Experiments} is devoted to numerical illustrations of the singleton error rates. Observations and discussions conclude, and the interested reader will find more details on conformal prediction and singleton validity in the appendices.

\section{Classification with reject option}\label{sec:rejectIntro}

In ML with reject option, the label space is extended to include a new label $\circledR$, which indicates a rejection. Formally, a model with rejection $m:\boldsymbol{X}\to \boldsymbol{Y}\cup\{\circledR\}$ is represented by a pair $(h,r)$ where $h:\boldsymbol{X}\to \boldsymbol{Y}$ is the predictor and $r:\mathcal{R}\to\mathbb{R}$ is the rejector. Here $\mathcal{R}$ may be for example $\boldsymbol{X}$ so that $r$ takes examples as input, the unit interval so that $r$ takes a confidence level as input, or even both. When the rejector determines that a prediction is likely to be wrong, $m$ outputs $\circledR$. Otherwise, the predictors output is accepted:
\begin{equation}
    \label{eq:defaultModelWithReject}
    m(x) 
    = 
    \begin{cases}
    $\circledR$ & \text{if the prediction is rejected}\\
    $h(x)$ & \text{if the prediction is accepted.}
    \end{cases}
\end{equation}
The rejector can use any available information as input, such as estimated probabilities, confidence levels, the object to be classified or any combination of these. Often, a reject threshold $t$ is specified, and a rejection depends on the value of $r$.
A concrete example, similar to the one described in \citet{HANCZAR2019106984}, could be a situation with $d$ classes, so that $\boldsymbol{Y}=\{c_1,\dots,c_d\}$, and $h(x) = \text{argmax}_i\{\pi_i(x)\}$, where $\pi_i$ is a probability estimate for the class $c_i$. We take $r(x) = \max_i\{\pi_i(x)\}$, and the corresponding model with reject option is then
\begin{equation}
    \label{eq:exampleModelWithReject}
    m(x) 
    = 
    \begin{cases}
    $\circledR$ & \text{if $r(x) < t$}\\
    $h(x)$ & \text{otherwise}
    \end{cases}
\end{equation}
The motivation for model \eqref{eq:exampleModelWithReject} is that if the largest probability estimate is large enough, it should make us trust the output more, and the error rate should decrease. However, the price we pay, is that if $t$ is large, the model will reject more objects. This means that there is a tradeoff between the reject rate and the error rate that has to be balanced. A common visualisation of the performance of a model with reject option, is to plot the error-reject curve. This, and some other visualisations, is described in detail in \citet{HANCZAR2019106984}. 

As discussed in \citet{hendrickx2024machine}, there are, at a high level, four principal reasons why a model exhibits large uncertainty in its predictions:
\begin{enumerate}
    \item 
    There can be objects $x\in\boldsymbol{X}$ that are associated with multiple labels in $\boldsymbol{Y}$. This can arise in situations where classes overlap.
    \item 
    Some examples in the training data are incorrect, e.g. labelling errors.
    \item 
    A shift in distribution between the training set and deployment.
    \item 
    An object $x\in\boldsymbol{X}$ could simply be inherently rare, thus making it difficult to classify.
\end{enumerate}
Based on these intuitions, rejections can be separated into two distinct kinds.

\textbf{Ambiguity rejection} happens when the object $x$ falls within a region where the label $y$ is ambiguous (corresponding to 1 and 2 above). Commonly, this means that $x$ falls close to the decision boundary.

\textbf{Novelty rejection} occurs when $x$ falls in a region where there was little or no training data (corresponding to 3 and 4 above). 

Classification with reject option goes back to \citet{chow1957optimum}, and it is well known that the optimal rejector is a version of Bayes rule: it rejects samples whose maximum posterior is below a threshold $1-\lambda_r$, which is known as Chows rule \citep{PILLAI20132256}. With modern deployment of ML systems in high-stakes, real-world applications, reject option has gained renewed attention \citep{shah2020online} and \citep{lei2020new}. 
In safety-critical applications, such as medical diagnoses, a classifier could handle the easy cases, thus freeing up human resources, while deferring the difficult ones to a human by rejecting them. 
Interesting developments include learning reject rules \citep{zhang2023survey} and partial rejection \citep{Karlsson2024} and \citep{LECAPITAINE2014843}, the latter  being closely connected to conformal prediction, which we introduce in Section \ref{section:A Brief Introduction to Conformal Prediction}.

\section{A Brief Introduction to Conformal Prediction}\label{section:A Brief Introduction to Conformal Prediction}

Conformal prediction (CP), developed by Vovk, Gammerman and Shafer \citep{alrw2}, is a method to generate prediction sets for essentially any predictive model, with a user specified confidence level $1-\eps$. These prediction sets are valid not just in the asymptotic sense that the error frequency in the long run is $\eps$, but also in the sense that for a smoothed CP, the error probability is exactly $\eps$, independently for each trial. 

Since its introduction in the late 90s, there has been an increasing interest in CP as a distribution-free method for uncertainty quantification, with several special issues of journals devoted to the topic, including Pattern Recognition \citep{GAMMERMAN2022108561} and Neurocomputing \citep{GAMMERMAN2020264}. The motivation is similar to that of reject option. With increasing deployment of ML in high-stakes application, uncertainty quantification can help guide decision-making, and as we shall see, CP can be used as a reject option. There are several excellent, shorter introductions to the field of CP, which include \citet{TOCCACELI2022108507} and \citet{shafer2008tutorial}. While the present work will only consider the basic form of CP, that produces valid prediction sets, the framework also includes probabilistic classification (Venn predictors) \citep{vovk2014venn}, and regression (Conformal Predictive Systems) \citep{VOVK2022108536}. 

This section introduces what is known as full CP, which is an online procedure. We consider only the case of classification, although CP can be used for regression as well. Throughout this work, we follow the notation and conventions used in \citet{alrw2}. Assume that Reality outputs a sequence of examples
\begin{equation}
    (x_1,y_1),(x_2,y_2),\dots
\end{equation}
where $x_i\in \boldsymbol{X}$ is called an object and $y_i\in \boldsymbol{Y}$ is its label. For convenience, we write $z_i = (x_i,y_i)$ which is an element of the example space $\boldsymbol{Z}:= \boldsymbol{X}\times \boldsymbol{Y}$. We assume that the examples are drawn from a distribution $P$ on $\boldsymbol{Z}^{\infty}$, and our standard assumption is that the distribution $P$ is exchangeable. Intuitively, exchangeability means that any permutation of the examples is equally probable. The exchangeability assumption is similar to the standard i.i.d. assumption, but is slightly weaker. 

\subsection*{Nonconformity Measures}
In order to construct a conformal predictor, one needs to define a nonconformity measure $A$. Informally, this is a function that measures how “strange” a new example is, given what we have seen before. Examples of nonconformity measures include the residuals of an ML model trained on all the past examples. 
Given a sequence of examples $z_1,\dots,z_n$, we form the bag (or multiset) $\bagname := \lbag z_1,\dots,z_n\rbag$. The bag $\bagname$ can be seen as a summary of our examples in the sense that it disregards the order in which they were observed. The number
\begin{equation}
    \alpha_i := A(\Lbag z_1,\dots,z_n\Rbag, z_i)
\end{equation}
is called the nonconformity score of the new example $z_i$. 

\subsection*{p-values}

The nonconformity score $\alpha_i$ by itself does not tell us how unusual $z_i$ is with respect to the other examples. For this, we need to compare it to the other nonconformity scores. A convenient comparison is to compute the fraction
\begin{equation}
    \frac{|\{j=1,\dots,n:\alpha_j\geq\alpha_i\}|}{n},
\end{equation}
i.e. the number of examples that are less conforming than $z_i$ divided by the total number of examples. This fraction, which is a rational number between $1/n$ and 1 is called the p-value of example $z_i$. A small p-value indicates that $z_i$ is very nonconforming, and a large p-value indicates that it is typical. To get the strongest validity guarantees, we need to handle the cases when $\alpha_j=\alpha_i$ more carefully by introducing some randomisation. This leads to the notion of smoothed conformal predictors. For a description of this, see \ref{section:More Details on Conformal Predictions}.

The conformal predictor determined by a nonconformity measure $A$ is 
\begin{equation}
    \Gamma^{\eps}(z_1,\dots,z_{n-1}, x_n):=\Gamma^{\eps}_n :=\left\{y\in \boldsymbol{Y}:\frac{|\{i=1,\dots,n:\alpha_i\geq\alpha_n\}|}{n} > \eps \right\}
\end{equation}
where
\begin{equation}
    \begin{aligned}
        \alpha_i &:= A(\bagname,(x_i,y_i)), & i=1,\dots,n-1\\
        \alpha_n &:= A(\bagname,(x_n,y)), & y\in \boldsymbol{Y}\\
        \bagname &= \Lbag (x_1,y_1),\dots,(x_{n-1},y_{n-1}), (x_n,y)\Rbag.
    \end{aligned}
\end{equation}
The procedure producing the prediction set $\Gamma^{\eps}$ may be viewed as for each possible label $y\in \boldsymbol{Y}$, testing the hypothesis that $y_n=y$. This is done by tentatively setting $y_n=y$ and computing the resulting p-value. The prediction set consists of all labels which result in a p-value larger than a specified significance level $\eps$.

It is important to note that the output of a conformal predictor is a subset of the label space, rather than a single point. As hinted at before, this can be viewed as a partial rejection \citep{Karlsson2024}. Predicting $\Gamma_n^{\eps}$ is the same as rejecting $\boldsymbol{Y}\backslash\Gamma_n^{\eps}$. The notion of error is therefore somewhat unusual compared with classical ML methods. Let $\Gamma$ be a confidence predictor, processing the infinite data sequence
\begin{equation}
    \label{eg:omegaIntro}
    \omega = (x_1,y_1,\dots)
\end{equation}
The event that $\Gamma$ makes an error at trial $n$, at the significance level $\eps$, i.e., $y_n\not\in\Gamma^{\eps}_n$, can be represented by a number 
\begin{equation}
    \err_n^{\eps}(\Gamma, \omega) :=
    \begin{cases}
        1 & \text{if $y_n\not\in \Gamma_n^{\eps}$} \\
        0 & \text{otherwise.}
    \end{cases}
\end{equation}
The total number of errors committed by $\Gamma$ after $n$ trials is then 
\begin{equation}
\Err^{\epsilon}_n(\Gamma,\omega):=\sum^n_{i=1}\text{err}^{\epsilon}_i(\Gamma,\omega).
\end{equation}
The number $\err_n^{\eps}(\Gamma, \omega)$ is the realisation of a random variable $\err_n^{\eps}(\Gamma)$.
We say that a confidence predictor $\Gamma$ is exactly valid if, for each $\eps$, 
\begin{equation}
    \err_1^{\eps}(\Gamma), \err_2^{\eps}(\Gamma),\dots
\end{equation}
is a sequence of independent Bernoulli random variables with parameter $\eps$. In words, the event of making an error is like getting heads when tossing a biased coin where the probability of getting heads is always $\eps$.

The main result regarding conformal predictors is the following proposition, whose proof can be found in \citep{alrw2}.

\begin{proposition}
    \label{prop:smoothedCPValidity}
    Any smoothed conformal predictor is exactly valid.
\end{proposition}
In other words, the prediction sets output by a smoothed CP will contain the true label with probability $1-\eps$ independently for each trial.
A self-contained introduction to CP in different settings, including more details on exchangeability, can be found in \ref{section:More Details on Conformal Predictions}. 
For a full discussion of conformal predictors, see \citet{shafer2008tutorial} and \cite{alrw2}. 

\subsubsection*{Confidence and Credibility}

The confidence of an object $x_i$ is denoted $1 - \confidence_{x_i}$, and is defined as one minus the smallest $\eps$ such that $\Gamma^{\eps}_i$ is a singleton set, i.e.
\begin{equation}
    \label{eq:pintwiseConfidence}
    1 - \confidence_{x_i} := \sup\{1-\eps:|\Gamma_i^{\eps}|=1\}.
\end{equation}
Similarly the credibility, $\credibility_{x_i}$ of $x_i$ is the smallest $\eps$ for which $\Gamma^{\eps}$ is empty,
\begin{equation}
    \label{eq:pointwiseCredibility}
    \credibility_{x_i} := \inf\{\eps:\Gamma_i^{\eps} = \varnothing\}.
\end{equation}
The confidence is equal to the second largest p-value associated with $x_i$, and the credibility is equal to the largest p-value.
    
\section{Conformal Predictors as models with reject option}\label{section:Conformal Predictors as Rejectors}
We have already seen that CP prediction sets may be viewed as a partial rejection (everything not included in the prediction set is rejected). In binary classification, the prediction sets output by a CP can contain one, two or no labels, but we really only make a clear prediction when they contain exactly one label. The other two situations contain little useful information; in terms of partial rejection, empty sets reject everything, and a two-label prediction simply states the tautology that the true label is a member of the label set. We may thus regard both empty and double predictions as a rejection. But the cause for rejection is different. In \citet{hendrickx2024machine}, the authors distinguish between novelty and ambiguity rejection, which we discuss in Section \ref{sec:rejectIntro}. An empty prediction set is rejected because we are not sufficiently confident in any label. This is novelty rejection. A double prediction set is rejected because we are too confident in predicting both labels. This is ambiguity rejection. We are unable to confidently distinguish between labels. So conformal prediction may be seen as applying two rejectors. Formally we augment the label space $\boldsymbol{Y}$ with two new symbols, $\circledR_{\varnothing}$, indicating a novelty rejection, and $\circledR_{\mathcal{D}}$, indicating an ambiguity rejection. Our model with reject option is then the CP rejector
\begin{equation}
    \label{eq:CPModelWithReject}
    m(x, \eps) 
    = 
    \begin{cases}
        \circledR_{\varnothing} & \text{if $|\Gamma^{\eps}(x)|=0$}\\
        \Gamma^{\eps}(x) & \text{if $|\Gamma^{\eps}(x)|=1$}\\
        \circledR_{\mathcal{D}} & \text{if $|\Gamma^{\eps}(x)|=2$}
    \end{cases}
\end{equation}
with a slight abuse of notation (we actually predict the single element in $\Gamma^{\eps}$ rather than the set). 
In this formulation, $\eps\in(0,1)$ is a reject parameter that determines which objects are rejected. In particular, we accept a prediction if and only if $\eps\in I_x:=[\confidence_x, \credibility_x)$ (see \eqref{eq:pintwiseConfidence} and \eqref{eq:pointwiseCredibility}). 

Conformal prediction provides validity guarantees in the sense that the true label is contained in the prediction set with probability $1-\eps$, but in a reject scenario we want to be more precise. Accepting only objects with sufficiently high confidence and credibility, i.e. the singleton predictions, we want to be able to state the error rate of the accepted predictions.
The next section deals with the question of error rate for the predictions we actually make in this scenario; the singleton prediction sets.

\section{Error Rate of the Accepted Predictions}\label{section:Error Rate}

This section discusses the error rate of the accepted predictions. What is the probability of $\err_i^{\eps}(\Gamma, \omega) = 1$ if we have observed $|\Gamma_i^{\eps}|=1$? The rationale behind this question is that, in terms of CP, a prediction set that covers all of $\boldsymbol{Y}$ is always correct, and an empty prediction set is always an error. In binary classification, the only remaining possibility is when the prediction set is a singleton, and thus it makes sense to ask for the probability that the singleton predictions are in error.
In essence, what we need to compute is the conditional probability $\Prob(\err_i^{\eps}~\big|~ |\Gamma_i^{\eps}|=1)$. 

\subsection{Online Smoothed Conformal Prediction}

Consider the cardinality of a prediction set, $|\Gamma_i^{\eps}|$, as a random variable that takes values in $\{0,1,2\}$. Moreover, let $\err$ be the event $\err_i^{\eps}(\Gamma, \omega)=1$, $E$ the event $|\Gamma_i^{\eps}|=0$, $S$ the event $|\Gamma_i^{\eps}|=1$ and $D$ the event $|\Gamma_i^{\eps}|=2$. Then $\{E,S,D\}$ is a partition of the sample space of $|\Gamma_i^{\eps}|$. 
By Proposition \ref{prop:smoothedCPValidity}, the probability of making an error (in the CP sense) is $\eps$, and by the law of total probability,
\begin{equation}
    \label{eq:totalProbability}
    \eps = \Prob(\err) =
    \Prob(\err~|~E)\Prob(E) + \Prob(\err~|~S)\Prob(S) + \Prob(\err~|~D)\Prob(D) = \Prob(E) + \Prob(\err~|~S)\Prob(S)
\end{equation}
since the error probabilities conditioned on $E$ and $D$ are known to be 1 and 0 respectively. Our main result follows immediately.

\begin{proposition}
    \label{prop:exactSingletonErrorOnline}
    Let $\Gamma$ be a smoothed conformal predictor in the online setting that processes the data sequence \eqref{eg:omega}. The probability of a singleton prediction making an error is
    \begin{equation}
        \label{eq:sigma}
        \sigma := \frac{\eps - \Prob(E)}{\Prob(S)}.
    \end{equation}
\end{proposition}
\begin{proof}
    Rearranging \eqref{eq:totalProbability} leads to \eqref{eq:sigma}. The condition $\Prob(E)\leq\eps\leq \Prob(E) + \Prob(S) = 1 - \Prob(D)$ ensures that $\sigma\in[0,1]$. Note however that the p-values of the correct label are uniformly distributed on $[0,1]$ (Theorem 11.2 \citep{alrw2}). Hence $\Prob(E) \leq \eps$ and $\Prob(D) \leq 1 - \eps$ for all $\eps\in(0,1)$, ensuring that both bounds are always satisfied.
\end{proof}
While this proposition tells us exactly what the singleton error probability is, it is not of much practical use since we do not know the probabilities of making empty and singleton predictions. Note, however, that by the law of large numbers,
\begin{align*}
    \lim_{n\to\infty}\frac{e}{n} = \Prob(E)\\
    \lim_{n\to\infty}\frac{s}{n} = \Prob(S)
\end{align*}
where $e$ is the number of empty predictions and $s$ is the number of singleton predictions. It follows that 
\begin{equation}
    \label{sigmaAsymptotic}
    \sigma = \lim_{n\to\infty}\frac{n\eps - e}{s}.
\end{equation}
In the online setting, this does not help much since Reality reveals the true label before we move on to make a new prediction. Nevertheless, if we insist on being ignorant, we may look away as Reality outputs $y_i$. After making $n$ such user-blind predictions, $(n\eps - e)/s$ is an approximation of $\sigma$. 

In light of Proposition \ref{prop:exactSingletonErrorOnline} we can now state the error rate and rejection rate of our rejector \eqref{eq:CPModelWithReject}. The error rate for a fixed $\eps$ is $\sigma$, and the reject rate is $1-\Prob(S)$.

\begin{remark}
    We called $\eps$ our reject parameter in \eqref{eq:CPModelWithReject} which is not consistent with the use of the term in \citep{chow1970optimum}. It does however fill the essential function of a reject parameter since the reject rate and error rate are functions of $\eps$ which justifies the use of the term here. 
\end{remark}

\section{Numerical Examples}\label{section:Numerical Experiments}

We will now turn our attention to a few examples, illustrating $\sigma$ for the cases covered earlier. The cases that we will look at are transductive conformal prediction in an online setting, inductive conformal prediction in an offline setting and inductive conformal prediction in an offline batch setting \textcolor{black}{(see Appendix \ref{ICP theory} for a description of inductive conformal prediction).}

\textcolor{black}{Transductive CP in the online setting can be quite computationally complex. In the batch mode, inductive CP simply involves training a machine learning model, computing the nonconformity scores of a calibration set, and sorting them. Assuming the calibration set is of size $m$, this can be done in $\mathcal{O}(m\log{m})$ once the base model has been trained.}

For the purpose of these illustrations, three different datasets will be used. More details are provided for the datasets in Table~\ref{tab:datasets}.
\begin{table}[h!]
\centering
\begin{tabular}{l|c|c}
Name & Nr.~examples & Nr.~features \\\hline
\texttt{qsar-biodeg} & 1055 & 41\\
\texttt{spambase} & 4601 & 57\\
\texttt{California-Housing-Classification} & 20640 & 8\\
\end{tabular}
\caption{\label{tab:datasets}Some details on the datasets used.}
\end{table}

\subsection{Full Conformal Prediction}
We will start by explaining our interpretation of \textit{Full Conformal Prediction} (FCP). In our interpretation of FCP, we will use a training set to predict one object at a time, taking into account how well the different possible labels for the object conforms to the training set. When a batch of objects is to be predicted, we will not learn the true labels of the predicted object and update our training set. In this setup, the order of the objects to be predicted is not important and will not affect the prediction sets of each individual object. We may as well describe an FCP as a single object Transductive Conformal Predictor (see Section 4.5 \citet{alrw2}). 

Here, we will use a Mondrian one nearest neighbour approach. The non-conformity measure is defined as
\begin{equation}
    A((x_1,y_1),\dots,(x_{n-1},y_{n-1}),(x_n,y)):=\min_{i,j \in I_y, i\ne j} d(x_i,x_j),
\end{equation}
where $d(\cdot,\cdot)$ is the Euclidean distance between two objects and $I_y = \{i=1,\dots,n : y_i = y\}$. An example is considered conforming if the distance between its object to the nearest object with the same label is small. The data set used is \texttt{qsar-biodeg} and the initial training set size is 100, and we predict 200 objects. If we run this 100 times, reshuffling the data before each run, we get the results, with statistical variations, shown in Figure~\ref{fig:tcp biodeg}. \textcolor{black}{Note that the curve in Figure \ref{fig:reject-tcp} indicates two different error rates for some reject rates. This happens because some $\varepsilon$ yield the same number of singleton predictions in this case. In this situation, it is prefereable to choose the $\varepsilon$ that leads to the lower error rate, which will always be the smallest one (the curve is parameterised by $\varepsilon$, and goes form the lower right corner to the upper right corner). The achievable reject rates depend on the specific conformal predictor, in particular on the distribution of p-values in the set $[0,1]\times[0,1]$. The highest proportion of singleton predictions for the one nearest neighbour CP on this dataset is approximately 0.4, which means that the lowest achievable reject rate (indicated in Figure \ref{fig:reject-tcp}) is approximately 0.6. The stacked fraction in Figure \ref{fig:sigma-tcp} indicate the proportion of empty, double and singleton predictions for each class that result form a fixed $\varepsilon$.} We also remark that the application of an FCP probably is not of practical interest, but it follows our theoretical reasoning for the singleton error rate, $\sigma$. \textcolor{black}{Computing the prediction set for a test object $x_n$ involves computing all distances ($\mathcal{O}(n^2)$) and sorting nonconformity scores ($\mathcal{O}(n\log{n})$). Note that this is the complexity for one prediction.}

\begin{figure}[h!]
\centering
\begin{subfigure}[b]{1.0\textwidth}
\centering
\includegraphics[width=0.8\linewidth]{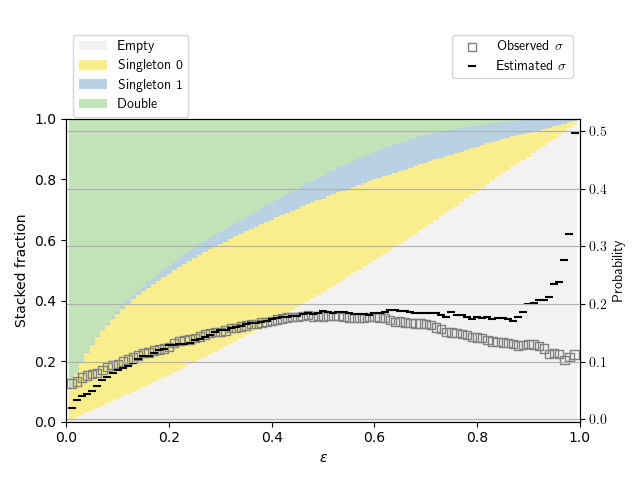}
\caption{This shows the fraction of prediction set sizes and $\sigma$ as a function of $\varepsilon$}
\label{fig:sigma-tcp}
\end{subfigure}
\vskip\baselineskip 
\begin{subfigure}[b]{1.0\textwidth}
\centering
\includegraphics[width=0.8\linewidth]{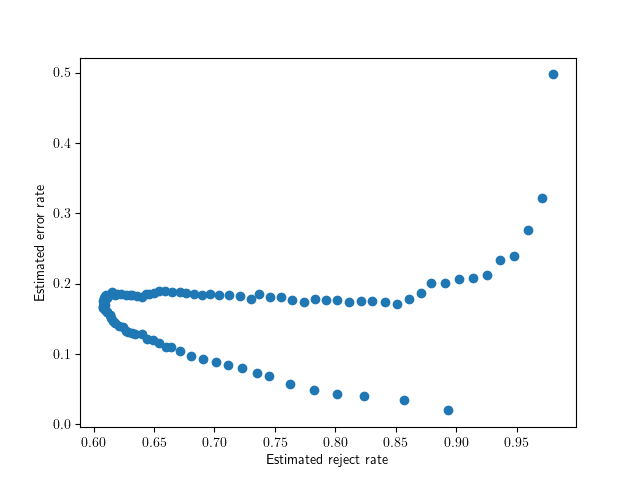}
\caption{\textcolor{black}{Error rate as a function of reject rate, parameterised by $\varepsilon$. The curve goes from the lower right corner to the upper right corner, with $\varepsilon$ increasing in the direction indicated by the arrow.}}
\label{fig:reject-tcp}
\end{subfigure}
\caption{Singleton error rates and error reject curve in the FCP case. The data set used is \texttt{qsar-biodeg}}
\label{fig:tcp biodeg}
\end{figure}

\subsection{Offline Inductive Conformal Prediction}

In this case, we are using the software package \texttt{Crepes}~\citep{crepes} to do the offline \textit{Inductive Conformal Prediction}, offline (ICP) (see \textcolor{black}{Appendix} \ref{section:More Details on Conformal Predictions}) using \texttt{RandomForest} from~\citet{scikit-learn} with the default \texttt{Crepes} nonconformity measure, and using the data set \texttt{spambase}. We are doing the experiments with proper training set size of 500, calibration set size of 500, and test set size of 100. This is done 1000 times, reshuffling the data set before selecting the different sets. The results are illustrated in Figure~\ref{fig:ICP spambase}.
\begin{figure}[h!]
\centering
\begin{subfigure}[b]{1.0\textwidth}
\centering
\includegraphics[width=0.8\linewidth]{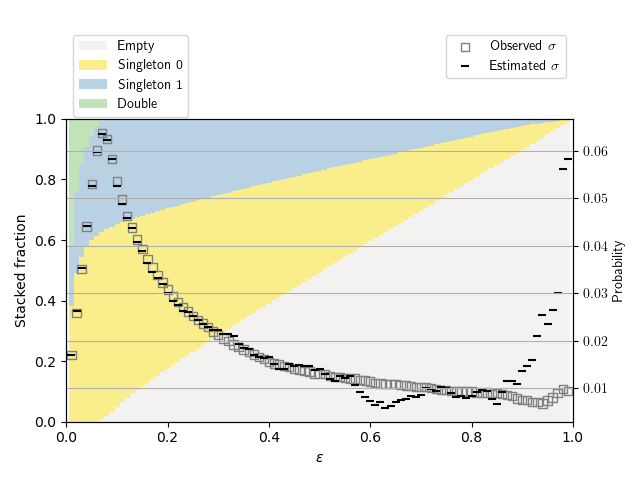}
\caption{This shows the fraction of prediction set sizes and $\sigma$ as a function of $\varepsilon$}
\label{fig:sigma-icp}
\end{subfigure}
\vskip\baselineskip 
\begin{subfigure}[b]{1.0\textwidth}
\centering
\includegraphics[width=0.8\linewidth]{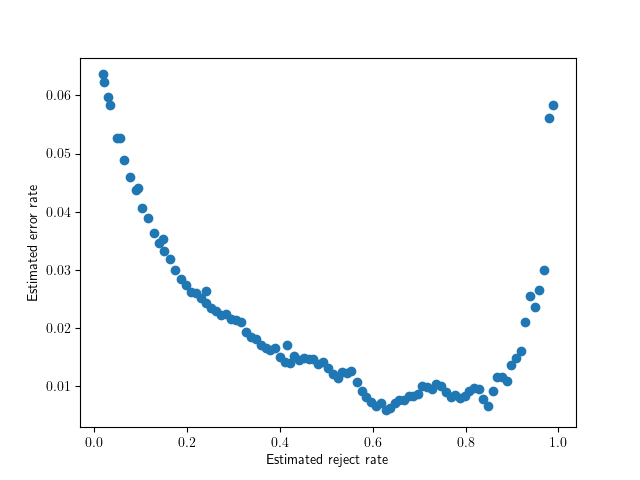}
\caption{This shows the error rate as a function of reject rate, parameterised by $\varepsilon$}
\label{fig:reject-icp}
\end{subfigure}
\caption{Singleton error rates and error reject curve in the offline ICP case for \texttt{spambase}.}
\label{fig:ICP spambase}
\end{figure}

\subsection{Batch Offline Inductive Conformal Prediction}
Here we use exactly the same setup as in the previous section, except that for the objects predicted we will learn their respective labels, a batch at a time, and add them to the training set. This will be repeated for ten times. The proper training set size will initially be 200 examples, but will grow with 100 for each batch, and  the calibration set will be fixed at a size of 300 examples. The data set used is \\ \texttt{California-Housing-Classification}. The division between proper training examples and calibration examples will be done randomly each time a new offline ICP model is trained. The results are illustrated in Figure~\ref{fig:batchICP CaliforniaHousing}.
\begin{figure}[h!]
\centering
\begin{subfigure}[b]{1.0\textwidth}
\centering
\includegraphics[width=0.8\linewidth]{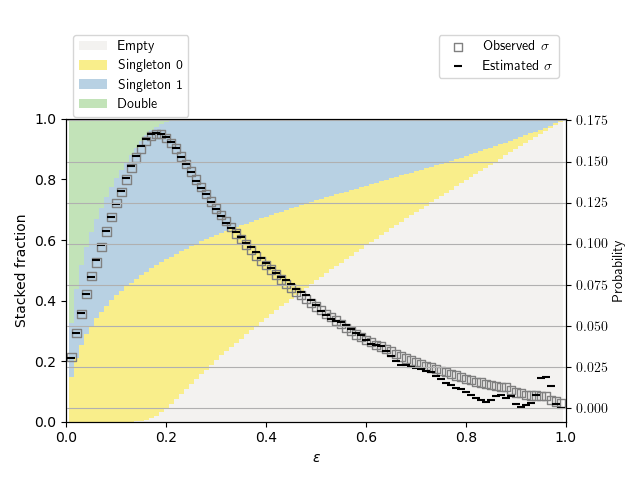}
\caption{This shows the fraction of prediction set sizes and $\sigma$ as a function of $\varepsilon$.}
\label{fig:sigma-icp-batch}
\end{subfigure}
\vskip\baselineskip 
\begin{subfigure}[b]{1.0\textwidth}
\centering
\includegraphics[width=0.8\linewidth]{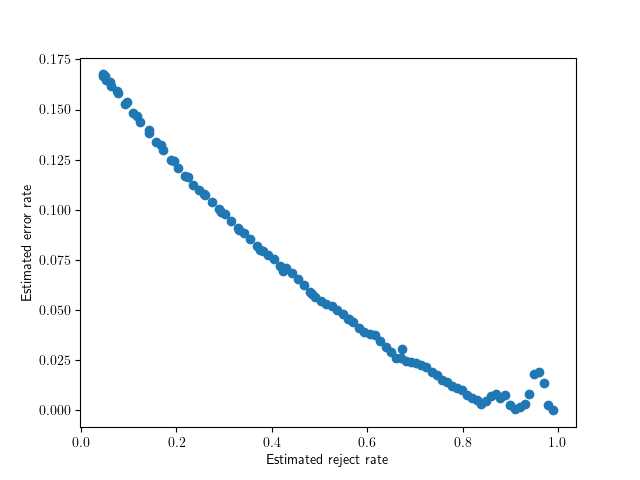}
\caption{This shows the error rate as a function of reject rate, parameterised by $\varepsilon$.}
\label{fig:reject-batch}
\end{subfigure}
\caption{Singleton error rates and error reject curve in the batch ICP case for \texttt{California-Housing-Classification}.}
\label{fig:batchICP CaliforniaHousing}
\end{figure}

\clearpage
\section{Conclusions}\label{section:Observations and Discussion}

We have formally presented a theoretical probability for the error of a singleton prediction, and given a way to approximate it in practice. The expression has appeared before in~\citet{linusson2016reliable} and~\citet{bortolussi2019conformal}. However, these previous works did not take into account the relevant adjustments needed to account for the inductive conformal predictor. Our main contribution is the proof of the singleton error rate, which translates to the error rate \textcolor{black}{and} the classifier with reject option. The formula in Proposition \ref{prop:exactSingletonErrorOnline} holds only in the online setting. The singleton error rates in other settings, can be found in \textcolor{black}{Appendix} \ref{section:More Details on Error Rate of Singletons}, and are proved analogously, but taking into account the overall error rate in the relevant setting, as described in \textcolor{black}{Appendix} \ref{section:More Details on Conformal Predictions}. With these results, models with reject option can be created in both the online and offline setting, with guaranteed distribution-free error rates.

In our numerical illustrations, our approximation of $\sigma$ is $(n\eps-e)/s$. The plots indicate that this approximation is noisy for large $\eps$, which likely is caused by a large number of traditional CP empty predictions in combination with a small number of accepted predictions. 

The theoretical error can only be known exactly if we know the true data-generating distribution, but the empirical distribution is in practice typically close to the true one, which leads to a useful contribution in practical applications. E.g. in a typical use case, one would predict many objects at once, using an inductive conformal predictor without knowing the true labels. Then $\hat{\sigma}(\eps) := (n\eps - e)/s$ is an unbiased estimator of the error rate for each $\eps\in(0,1)$. A user could easily compute $\hat{\sigma}(\eps)$ for any $\eps$ directly from the p-values, plotting the estimated error rates as in our figures to guide decisions about suitable reject rates.

The significance level, $\varepsilon$, can be used to set a reject threshold along the lines of~\citet{chow1970optimum}, since we can have a scenario where we only accept singleton predictions. However, the actual limits for the reject rate and the error rate will depend on the conformal predictor and the data set used, as is seen in our results. The error-reject rate curves can potentially be useful in practical applications when a user needs to define a tolerable error rate or reject rate. \textcolor{black}{Depending on the specific conformal predictor, not all reject rates are achievable. Consider e.g. Figure \ref{fig:reject-tcp}, where the lowest possible reject rate is approximately 0.6. Moreover, the same reject rate may lead to different error rates, which follows from two different significance levels leading to the same number of singleton predictions. In such a case, it is preferable to choose the significance level that yields the lowest error rate. In other cases, such as Figure \ref{fig:ICP spambase}, the error-reject curve shows that very low error rates are achievable, but perhaps at an unacceptably high reject rate. Depending on the use case, a rejected prediction may trigger different actions, which lead to different desiderata. If, say, rejected predictions lead to costly expert evaluation, and the consequences of errors are not catastrophic, the maximum tolerable reject rate may be quite low. On the other hand, if lower error rates are desirable regardless of the reject rate, the curves allow users to identify the lowest achievable error rate, and know what reject rate they have to accept to achieve it. In this work, we have only considered the case of binary classification. This is no fundamental restriction, as} multiclass problems (with more than two possible labels) can be handled by building several models in a one-against-all fashion.

\appendix

\section{More Details on Conformal Predictions}\label{section:More Details on Conformal Predictions}

Conformal prediction (CP), developed by Vovk, Gammerman and Shafer, is a method to generate prediction sets for essentially any predictive model, with a user specified confidence level $1-\eps$. These prediction sets are valid not just in the asymptotic sense that the error frequency in the long run is $\eps$, but also in the sense that for a smoothed CP, the error probability is exactly $\eps$, independently for each trial. For a full introduction to conformal predictors, see \citet{shafer2008tutorial} and\citet{alrw2}. Conformal predictors are a type of confidence predictors which we introduce for completeness, but first some basic assumptions.

Assume that Reality outputs a sequence of examples
\begin{equation}
    \label{eq:RealityOutputs}
    (x_1,y_2),(x_2,y_2),\dots
\end{equation}
where $x_i\in \boldsymbol{X}$ is called an object and $y_i\in \boldsymbol{Y}$ is its label. The measurable sets $\boldsymbol{X}$ and $\boldsymbol{Y}$ are then called the object space and the label space, respectively. For convenience, we may write $z_i:=(x_i,y_i)$ for an example, and consider it an element of the example space $\boldsymbol{Z}:=\boldsymbol{X}\times \boldsymbol{Y}$. For the entirety of this paper, we will be concerned with binary classification, that is $|\boldsymbol{Y}|=2$. The standard assumption in ML is that Reality chooses the examples independently from some probability distribution $Q$ on $\boldsymbol{Z}$; the examples are independent and identically distributed (i.i.d.). Most results for conformal predictors hold under the weaker assumption that the infinite data sequence \eqref{eq:RealityOutputs} is drawn from an exchangeable distribution $P$ on $\boldsymbol{Z}^{\infty}$. That $P$ is exchangeable means that for every positive integer $n$, every permutation $\pi$ of $\{1,\dots,n\}$, and every measurable set $E\subseteq Z^n$,
\begin{equation}
    \label{eq:ExchangeabilityDef}
    \begin{aligned}
        P\{(z_1,z_2,\dots)&\in \boldsymbol{Z}^{\infty}:(z_1,\dots,z_n)\in E\} \\
        &= P\{(z_1,z_2,\dots)\in \boldsymbol{Z}^{\infty}:(z_{\pi(1)},\dots,z_{\pi(n)})\in E\}.
    \end{aligned}
\end{equation}
In words, the distribution of the examples does not depend on their order. All results that we will state also hold in the finite-horizon setting, where Reality generates $N$ examples. The exchangeability assumption then simplifies to assuming that the distribution $P$ on $\boldsymbol{Z}^N$ has the property that for every permutation $\pi$ of $\{1,\dots,N\}$ and every measurable $E\subseteq \boldsymbol{Z}^N$
\begin{equation}
    \label{eq:ExchangeabilityFiniteDef}
    P(E)
    = P\{(z_1,\dots,z_N)\in \boldsymbol{Z}^N:(z_{\pi(1)},\dots,z_{\pi(N)})\in E\}.
\end{equation}
A confidence predictor $\Gamma$ is an algorithm that takes as input $x_1,y_1,\dots,x_{n-1},y_{n-1},x_n$ and a significance level $\eps$, and outputs a prediction set
$$
    \Gamma_n^{\eps} = \Gamma^{\eps}_n(x_1,y_1,\dots,x_{n-1},y_{n-1},x_n).
$$
We require also that the prediction sets shrink as $\eps$ increases, that is whenever $\eps_1\geq\eps_2$, $\Gamma^{\eps_1}_n \subseteq\Gamma^{\eps_2}_n$.

\subsection{Validity}
Let $\Gamma$ be a confidence predictor, processing the data sequence
\begin{equation}
    \label{eg:omega}
    \omega = (x_1,y_1,\dots)
\end{equation}
drawn from an exchangeable distribution where $x_i\in \boldsymbol{X}$. The event that $\Gamma$ makes an error on the $n$th trial, that is $y_n\not\in\Gamma_n^{\eps}$, can be represented by a number
\begin{equation}
    \label{eq:err}
    \err_n^{\eps}(\Gamma, \omega) :=
    \begin{cases}
        1 & \text{if $y_n\not\in \Gamma_n^{\eps}$} \\
        0 & \text{otherwise.}
    \end{cases}
\end{equation}
The total number of errors committed by $\Gamma$ after $n$ trials is then 
\begin{equation}
    \label{eq:Err}
\Err^{\epsilon}_n(\Gamma,\omega):=\sum^n_{i=1}\text{err}^{\epsilon}_i(\Gamma,\omega).
\end{equation}
The number $\err_n^{\eps}(\Gamma, \omega)$ is the realisation of a random variable $\err_n^{\eps}(\Gamma)$.
We say that a confidence predictor $\Gamma$ is exactly valid if, for each $\eps$, 
\begin{equation}
    \err_1^{\eps}(\Gamma), \err_2^{\eps}(\Gamma),\dots
\end{equation}
is a sequence of independent Bernoulli random variables with parameter $\eps$. In words, the event of making an error is like getting heads when tossing a biased coin where the probability of getting heads is always $\eps$.

\subsection{Full Conformal Prediction Online}

In the online setting, where Reality first outputs an object $x_n$, we produce $\Gamma^{\eps}_n$, and then immediately Reality outputs the true label $y_n$, which can then be used to produce $\Gamma^{\eps}_{n+1}$. While this procedure is both computationally intensive, and in practice not very useful; if we immediately see the true labels, predicting them beforehand makes little sense, it is still of theoretical interest, mainly because CP provides the strongest possible validity guarantees here. There are ways of taking CP offline and lighten the computational load dramatically, but the cost of this is that the validity suffers. In practice, however, the difference in validity is typically small.

\subsubsection{Nonconformity Measure}
A nonconformity measure is a way of scoring ho strange an example is, given all the previous examples we have seen. Formally, it is a measurable mapping
\begin{equation}
    A:\boldsymbol{Z}^{(*)}\times \boldsymbol{Z}\to\overline{\mathbb{R}}
\end{equation}
where $\boldsymbol{Z}^{(*)}$ is the set of all bags (or multisets) of elements of $\boldsymbol{Z}$. We write $\Lbag z_1,\dots,z_n\Rbag$ for the bag consisting of the elements $z_1,\dots,z_n$, some of which are allowed to be identical. A bag, in other words, is a set with repetition allowed. Given a sequence of old examples $z_1,\dots,z_n$, and a new example $\boldsymbol{Z}$, the number 
\begin{equation}
    A(\Lbag z_1,\dots,z_n\Rbag,z)
\end{equation}
is called the nonconformity score of $\boldsymbol{Z}$.

\subsubsection{p-values}
Given a nonconformity measure $A$ and a bag $\Lbag z_1,\dots,z_n\Rbag$, we can compute the nonconformity score
\begin{equation}
    \alpha_i := A(\Lbag z_1,\dots,z_n\Rbag,z_i)
\end{equation}
for each example $z_i$ in the bag. By itself, $\alpha_i$ does not tell us how unusual $z_i$ is with respect to the other elements in the bag. For this, we need to compare the nonconformity scores of all the examples. We do this by computing the fraction,
\begin{equation}
    \frac{|\{j=1,\dots,n:\alpha_j\geq\alpha_i\}|}{n},
\end{equation}
which we denote the p-value of $z_i$. If the fraction is small, then $z_i$ is very nonconforming (an outlier), and if it is large $z_i$ is very conforming. To get the strongest validity guarantees, we need to handle the borderline cases when $\alpha_j=\alpha_i$ more carefully. The smoothed p-value of $z_i$ is defined as
\begin{equation}
    \frac{|\{j=1,\dots,n:\alpha_j>\alpha_i\}| + \tau_n|\{j=1,\dots,n:\alpha_j=\alpha_i\}|}{n},
\end{equation}
for $\tau_n\in[0,1]$ uniformly distributed at random, and independent of everything else. The function that maps a sequence of examples to (smoothed) p-values in $[0,1]$ is called a (smoothed) conformal transducer;
\begin{equation}
    \label{eq:conformalTransducer}
    \begin{aligned}
        f(x_1,\tau_i,y_1,\dots&,x_{n-1},\tau_{n-1},y_{n-1}, x_n,\tau_n,y_n)\\
        &:= \frac{|\{j=1,\dots,n:\alpha_j>\alpha_i\}| + \tau_n|\{j=1,\dots,n:\alpha_j=\alpha_i\}|}{n}
    \end{aligned}
\end{equation}

\subsubsection{Conformal Predictors}

We are now in a position to define conformal predictors. For our purposes, we will be mainly interested in smoothed CPs, which utilise smoothed p-values. The smoothed conformal predictor determined by the nonconformity measure $A$ is the randomised confidence predictor, 
\begin{equation}
    \begin{aligned}
        \Gamma(x_1,\tau_i,y_1,&\dots,x_{n-1},\tau_{n-1},y_{n-1}, x_n,\tau_n) \\
        &= \{y\in \boldsymbol{Y} : \frac{|\{i=1,\dots,n:\alpha_i>\alpha_n\}| + \tau_n|\{i=1,\dots,n:\alpha_i=\alpha_n\}|}{n} > \eps\},
    \end{aligned}
\end{equation}
where
\begin{equation}
    \begin{aligned}
        \alpha_i &:= A(\bagname,(x_i,y_i), & i=1,\dots,n-1\\
        \alpha_n &:= A(\bagname,(x_n,y), &\\
        \bagname &= \Lbag (x_1,y_1),\dots,(x_{n-1},y_{n-1}), (x_n,y)\Rbag.
    \end{aligned}
\end{equation}
The main result regarding conformal predictors is the following proposition, whose proof can be found in \citep{alrw2}.
\begin{proposition}
    Any smoothed conformal predictor is exactly valid.
\end{proposition}
In other words, the prediction sets output by a smoothed CP will contain the true label with probability $1-\eps$ independently for each trial.

\subsection{Mondrian Conformal Predictors}

A Mondrian taxonomy is a function
\begin{equation}
    \label{eq:mondrianDef}
    \kappa:\mathbb{N}\times \boldsymbol{Z} \to \boldsymbol{K}
\end{equation}
where $\boldsymbol{K}$ is a measurable space of categories. $\kappa$ maps each pair $(n,x)$ to its category, and we require that the pre-image $\kappa^{-1}(k)$ of each category $k\in \boldsymbol{K}$ for a rectangle $A\times B$ for some $A\in\mathbb{N}$ and $B\in \boldsymbol{Z}$. Given a Mondrian taxonomy, we can define a Mondrian nonconformity measure
\begin{equation}
    \label{eq:mondrianNonconformity}
    A:K^*\times(Z^{(*)})^K\times Z\to\overline{\mathbb{R}},
\end{equation}
and the smoothed Mondrian conformal transducer (smoothed MCT) determined by $A$, which produces p-values
\begin{equation}
    \label{eq:smoothedMCT}
    \begin{aligned}
        p_n 
        &= f(x_1,\tau_1,y_1,\dots,x_n,\tau_n,y_n)\\
        &:= \frac{|\{i:\kappa_i=\kappa_n, \alpha_i>\alpha_n\}| + \tau_n|\{i:\kappa_i=\kappa_n, \alpha_i=\alpha_n\}|}{|\{\kappa_i=\kappa_n\}|}
    \end{aligned}
\end{equation}
where $i=1,\dots,n$, $\kappa_i:=\kappa(i,z_i)$, $z_i=(x_i,y_i)$ and
\begin{equation}
    \alpha_i:=A(\kappa_1,\dots,\kappa_n,(k\in \boldsymbol{K} \mapsto \Lbag z_j:j\in\{i,\dots,n\}, \kappa_j=k\Rbag)z_i)
\end{equation}
for $i=1,\dots,n$ such that $\kappa_i=\kappa_n$.

By Proposition 4.6 \citep{alrw2}, any smoothed MCT is category-wise exact with respect to $\kappa$, meaning that for all $n$, the conditional probability distribution of $p_n$ given $\kappa(1,z_1),p_1,\dots,\kappa(n-1,z_{n-1}),p_{n-1},\kappa(n,z_n)$ is uniform on $[0,1]$.

Given a (smoothed) MCT we can always produce a (smoothed) Mondrian conformal predictor (MCP)
\begin{equation}
    \label{eq:defMCT}
    \Gamma^{(\eps_k:k\in \boldsymbol{K})}(z_1,\dots,z_{n-1},x_n)
    := \{y\in \boldsymbol{Y}:f(x_1,\tau_1,y_1,\dots,x_n,\tau_n,y) > \eps_{\kappa(n,(x_n,y))}\}
\end{equation}
which is then category-wise exactly valid with respect to $\kappa$ in the sense that the error events for each category $k\in K$ are Bernoulli random variables with parameter $\eps_k$. 

\subsection{Inductive Conformal Predictors}\label{ICP theory}

Full conformal predictors are notoriously computationally heavy to implement. Inductive conformal predictors solve this practical issue. For details, we refer the reader to 4.2 in \citet{alrw2}.

\subsubsection*{Offline ICP}
In the offline mode, an ICP $\Gamma$ is obtained by splitting the training set of size $l$ into a proper training set of size $m$ and a calibration set of size $h:=l-m$. Then for each test object $x_i$, $i=l+1,\dots,l+k$, the prediction sets are
\begin{equation}
    \label{eq:offlineICP}
    \Gamma^{\eps}(z_1,\dots,z_l,x_i) := \bigg\{y\in \boldsymbol{Y}: \frac{|\{j=m+1,\dots,l:\alpha_j\geq\alpha_i\}|+1}{l-m+1} >\eps\bigg\}.
\end{equation}
where
\begin{equation}
    \label{eq:offlineAlpha}
    \begin{aligned}
        \alpha_j &:= A((z_1,\dots,z_m),z_j), & j=m+1,\dots,i-1\\
        \alpha_i &:= A((z_1,\dots,z_m),(x_i,y) &
    \end{aligned}
\end{equation}
Smoothed ICPs offline are defined as in the case of full conformal predictors, by introducing a random tie breaking rule for the nonconformity scores. For offline ICPs, it is still true that the probability of making an error is bounded by $\eps$, but the error events are unfortunately no longer independent. With a small modification, one can however ensure a modified notion of validity. After each test example $x_i$ is processed, we add its corresponding nonconformity score $\alpha_i$ to the pool of nonconformity scores used to generate the next prediction set. 

\subsubsection*{Semi-online ICP}
The semi-online ICP is thus
\begin{equation}
    \label{eq:semiOnlineCP}
    \Gamma^{\eps}(z_1,\dots,z_l,x_i) := \bigg\{y\in \boldsymbol{Y}: \frac{|\{j=m+1,\dots,i:\alpha_j\geq\alpha_i\}|}{i-m} >\eps\bigg\},
\end{equation}
where
\begin{equation}
    \label{eq:semiOfflineAlpha}
    \begin{aligned}
        \alpha_j &:= A((z_1,\dots,z_m),z_j), & j=m+1,\dots,l\\
        \alpha_i &:= A((z_1,\dots,z_m),(x_i,y) &
    \end{aligned}
\end{equation}
for an inductive nonconformity measure $A$ (see 4.4.2. in \citet{alrw2}). 

\subsubsection*{Validity of semi-online ICPs}

The semi-online ICP \eqref{eq:semiOnlineCP} is conservatively valid, and the corresponding smoothed semi-online ICP is exactly valid by proposition 4.1 \citep{alrw2}.

\subsubsection*{Validity of offline ICPs}

Generally, a confidence predictor is said to be $(\delta_n)$-conservative, where $\delta_1,\delta_2,\dots$ is a sequence of non-negative numbers, if for any exchangeable probability distribution $P$ on $\boldsymbol{Z}^{\infty}$ there exists a family 
\begin{equation}
    \label{eq:xiFamily}
    \xi_n^{(\eps)}, ~\eps\in(0,1), ~n=1,2,\dots
\end{equation}
of $\{0,1\}$-valued random variables on $(\boldsymbol{Z}^{\infty}\times[0,1], P\times \boldsymbol{U})$ such that
\begin{itemize}
    \item 
    for all $n$ and $\eps$, $\xi_n^{(\eps)}$ depends only on the first $n$ examples;
    \item 
    for fixed $\eps$, $\xi_1^{(\eps)}, \xi_2^{(\eps)},\dots$ is a sequence of independent Bernoulli random variables with parameter $\eps$:
    \item 
    for all $n$ and $\eps$, $\err_n^{(\eps-\delta_n)}(\Gamma)\leq \xi_n^{(\eps)}$.
\end{itemize}
Note that when $\delta_n=0$, $n=1,2,\dots$, this reduces to the definition of conservative validity. By proposition 4.3 \citep{alrw2}, the confidence predictor
\begin{equation}
    \label{eq:extendedOfflineICP}
    \Tilde{\Gamma}^{\eps}(z_1,\dots,z_{i-1},x_i) := 
    \begin{cases}
        \Gamma^{\eps}(z_1,\dots,z_l,x_i) & \text{if $i>l$}\\
        \boldsymbol{Y} & \text{otherwise},
    \end{cases}
\end{equation}
where $\Gamma^{\eps}(z_1,\dots,z_l,x_i)$ is defined by \eqref{eq:offlineICP}, is $(\delta_n)$-conservative where
\begin{equation}
    \label{eq:deltan}
    \delta_n :=
    \begin{cases}
        \frac{i-l}{l-m} & \text{if $i>l$} \\
        0 & \text{otherwise}.
    \end{cases}
\end{equation}

\subsection{Training-Conditional Validity of ICPs}

A set predictor $\Gamma$ is $(\varepsilon,\delta)$-valid for $\varepsilon,\delta\in(0,1)$ if, for any probability distribution $Q$ on $\boldsymbol{Z}$,
\begin{equation}
    \label{eq:defEpsilonDeltaValidity}
    Q^l\{(z,1,\dots,z_l):Q(\Gamma(z_1,\dots,z_l))\geq 1-\varepsilon\}\geq 1-\delta,
\end{equation}
where
\begin{equation}
    \label{gammaFunctionDef}
    \Gamma(z_1,\dots,z_l):=\{(x,y)\in \boldsymbol{Z}:y\in\Gamma(z_1,\dots,z_l,x)\}.
\end{equation}
This means essentially that with probability at least $1-\delta$, the probability of making an error is at most $\varepsilon$. 

Let $\Gamma$ be an offline ICP with training and proper training set sizes $l$ and $m$ respectively. By proposition 4.9 \citep{alrw2}, if we define
\begin{equation}
    \label{eq:epsilonPrime}
    \tilde{\eps} := \eps - \sqrt{\frac{\ln\frac{1}{\delta}}{2h}},
\end{equation}
where $h:=l-m$ is the size of the calibration set, the set predictor $\Gamma^{\tilde{\eps}}$ is $(\eps, \delta)$-valid.

\subsection{Batch Offline ICP}

In practical settings, one common scenario is that we have access to a training set of size $l$, and want to make predictions for, say, $N$ objects. We can do this using an offline ICP. After making the predictions, we may gather the true labels for our objects, and include the resulting examples in the training set, which can then be used to make $N$ more observation. We are in a batch update setting.

Formally, let $l_0:=l$ the training set size for batch 0. Then we choose a proper training set size $m_0=m$, and obtain a calibration set size $h_0:=l-m$. If we fix the batch size to be $N$, we see that the training set size for the batch $k$ is $l_k = l + kN$. Then we pick a proper training set size $m_k$ and the calibration set size is $h_k:=l_k - m_k$. In each batch, have that $\Gamma^{\tilde{\eps}}$ is $(\eps, \delta)$-valid, where $\tilde{\eps}$ is defined in \eqref{eq:epsilonPrime}. We are then faced with a choice about how to use the new $N$ examples obtained since the last batch. Keeping $\delta$ fixed, we can aim for decreasing $\tilde{\eps}$. Let $\tilde{\eps}_k$ be the $\tilde{\eps}$ required for $\Gamma^{\tilde{\eps}}$ to be $(\eps,\delta)$-valid in the $k$th batch. Then we see that 
\begin{equation}
    \label{eq:epsilonPrimeAsymptotic}
    \lim_{k\to\infty}\tilde{\eps}_k = \lim_{k\to\infty}\eps - \sqrt{\frac{\ln\frac{1}{\delta}}{2h}} = \eps,
\end{equation}
but the uncertainty in our validity is fixed at $(1-\delta)$. On the other hand, we could be happy with the correction in $\eps$ and instead aim to decrease the uncertainty. Letting $\delta_0:= \delta$ we may choose $\delta_k$ to be
\begin{equation}
    \label{eq:deltaUpdate}
    \delta_{k+1} = \delta_k^{q_{k+1}}
\end{equation}
where $q_{k+1}=\frac{h_{k+1}}{h_k}$, and keep $\tilde{\eps}$ fixed. To realise this, just solve the equation $\tilde{\eps}_{k+1} = \tilde{\eps}_k$ for $\delta_{k+1}$.
If we require $h_{k+1}>h_k$, since $\delta\in(0,1)$, we see that
\begin{equation}
    \label{eq:deltaLimit}
    \lim_{k\to\infty}\delta_k=0.
\end{equation}
Of course, we may be interested in decreasing both $\tilde{\eps}_k$ and $\delta_k$. Choosing $\delta_{k+1} < \delta_k^{q_{k+1}}$ accomplish this goal. Alternatively, we may choose to first decrease $\delta$ to some desired level, and only then start decreasing $\tilde{\eps}_k$ to ensure higher efficiency. Recall that a confidence predictor must satisfy the property of nested prediction sets (2.4) in \citep{alrw2}), so decreasing the significance level will increase the efficiency.

We may also choose to increase the proper training set size, which would also be expected to increase the efficiency of $\Gamma$. 

To summarise, in each batch $k$ we have two choices to make. The proper training set size, $m_k$ which will determine the calibration set size $h_k$, and the uncertainty parameter $\delta_k$ which will determine the batch significance level $\tilde{\eps}_k$. Of course, we do not need to keep the batch size fixed. If we let $N_0 = l$ we may choose batch sizes $N_1,N_2,\dots$ where $N_k\geq1$ for $k=1,2,\dots$, and everything above still holds.

\section{More Details on Error Rate of Singletons}\label{section:More Details on Error Rate of Singletons}

This section discusses the validity of singleton predictions in some other settings.

\subsection{Singleton Validity of Mondrian Conformal Predictors}

Recall that a smoothed MCP is category-wise exactly valid with respect to the Mondrian taxonomy $\kappa$. For the singleton errors, this means that in each category $k\in \boldsymbol{K}=\{0,1\}$, the probability of a singleton error is
\begin{equation}
    \label{eq:MondrianSingletonError}
    \sigma_k = \frac{\eps_k-\Prob(E_k)}{\Prob(S_k)}.
\end{equation}
\subsubsection*{Label-Conditional Validity}
In \citep{linusson2016reliable} it is shown that it is impossible to arrive at a label-conditional $\sigma$, since $\Prob(E_k)$ is unknown. They do however give a conservative estimate by assuming $\Prob(E_k)=0$. In practice, a better estimate $\Prob(E_k)$ could be to use the observed frequencies of $E_0$ and $E_1$ respectively as observed in the training set.

\subsubsection*{Object-Conditional Validity}
If, on the other hand, we want to use some other Mondrian taxonomy, the situation can be different. Consider for example the case when we try to predict whether a patient has Alzheimer's disease. It is known that about twice as many females have Alzheimer's disease than males, presumably due to the fact that females tend to live longer. This observation leads to at least two possible Mondrian taxonomies of interest, where the categories are determined by features of the objects, not the labels. We could either have just two categories, males and females, or we could have age related categories. In each case, crucially for us, we can observe the number of empty and singleton predictions we make. For each category $k$, we can observe $n_k$, the total number of predictions made for the category $k$, $e_k$ and $s_k$, the number of empty and singleton predictions we make for the category $k$. Then 
\begin{equation}
    \frac{n_k\eps_k-e_k}{s_k}
\end{equation}
is our estimate of the object-conditional singleton validity of our MCP, and it converges to $\sigma_k$ when $n_k\to\infty$ by the law of large numbers.

\subsection{Training-Conditional Singleton Validity ICPs}\label{sec:ICP singleton}

Suppose $\Gamma$ is an ICP. We know $\Gamma^{\tilde{\eps}}$ with $\tilde{\eps}$ defined in \eqref{eq:epsilonPrime} is $(\eps,\delta)$-valid. This means 
\begin{equation*}
    \Prob(\Prob(\err) < \eps)\geq1-\delta,
\end{equation*}
and we may use the law of total probability as before to deduce
\begin{equation}
    \Prob\bigg(\Prob(\err~|~S)<\frac{\tilde{\eps}-\Prob(E)}{\Prob(S)}\bigg)\geq1-\delta.
\end{equation}
Let us denote this upper bound on the singleton error probability by $\tilde{\sigma}$. In this setting, we typically use the same ICP to produce many prediction sets without retraining and/or recalibrating. Say we have a test set of size $n$, and wish to test the validity of our ICP on just the singleton predictions. We observe that we have $e$ empty predictions and $s$ singletons. Then
\begin{equation}
    \label{eq:sigmaPrimeLimit}
    \lim_{n\to\infty}\frac{n\tilde{\eps}-e}{s} = \tilde{\sigma}.
\end{equation}

\subsection{Batch Offline ICP}

We are, as always, interested in what can be said about the validity of the singleton predictions in the batch update scenario. In each batch we have an offline ICP, and we make $N$ predictions. Our approximation of $\sigma$ in \eqref{eq:sigmaPrimeLimit} still holds, but with $N$ in place of $n$. If $N$ is large, this approximation can be expected to be good. Explicitly,
\begin{equation}
    \label{eq:sigmaPrimeK}
    \tilde{\sigma}_k:=\frac{N_k\tilde{\eps}_k-e_k}{s_k}
\end{equation}
where $e_k$ and $s_k$ are the number of empty and singleton predictions respectively, is a good approximation of $\tilde{\sigma}$ if $N_k$ is sufficiently large. Moreover, if we use the batch updates to decrease $\tilde{\eps}_k$, we have that
\begin{equation}
    \label{eq:sigmaPrimeBatchLimit}
    \lim_{k\to\infty}\tilde{\sigma}_k = \sigma = \frac{\eps-\Prob(E)}{\Prob(S)},
\end{equation}
and if we have also chosen to decrease $\delta_k$, this is exactly the error probability of the singleton predictions in the limit. Looking closely at \eqref{eq:defEpsilonDeltaValidity} we see that this means that the probability of a singleton prediction being in error is at most $\sigma$ so that $\Gamma$ is conservatively valid in the limit.


\bibliographystyle{apacite}

\bibliography{main}



\end{document}